\def\year{2021}\relax
\documentclass[letterpaper]{article} % DO NOT CHANGE THIS
\usepackage{aaai21}  % DO NOT CHANGE THIS
\usepackage{times}  % DO NOT CHANGE THIS
\usepackage{helvet} % DO NOT CHANGE THIS
\usepackage{courier}  % DO NOT CHANGE THIS
\usepackage[hyphens]{url}  % DO NOT CHANGE THIS
\usepackage{graphicx} % DO NOT CHANGE THIS
\usepackage{natbib}  % DO NOT CHANGE THIS AND DO NOT ADD ANY OPTIONS TO IT
\usepackage{caption} % DO NOT CHANGE THIS AND DO NOT ADD ANY OPTIONS TO IT
\usepackage[utf8]{inputenc} % allow utf-8 input
\usepackage[T1]{fontenc}    % use 8-bit T1 fonts
\usepackage{booktabs}       % professional-quality tables
\usepackage{amsfonts,amsmath}       % blackboard math symbols
\usepackage{nicefrac}       % compact symbols for 1/2, etc.
\usepackage{microtype}      % microtypography
\usepackage[inline]{enumitem}
\usepackage{graphicx}
\usepackage{subcaption}
\usepackage{bbm}
\usepackage{booktabs}
\usepackage{booktabs,siunitx,caption}
\usepackage{amsthm}
\usepackage{amsmath}
\newtheorem{definition}{Definition}
\usepackage{mathtools}
\mathtoolsset{showonlyrefs}
\usepackage{enumitem}    
\usepackage{upgreek}
\usepackage{comment}

\usepackage{xcolor}
\usepackage{paralist}

\newcommand{\todo}[1]{}
\renewcommand{\todo}[1]{{\color{red} TODO: {#1}}}

\newcommand{\R}{\mathbb{R}}

\newtheorem{theorem}{Theorem}
\newtheorem{lemma}[theorem]{Lemma}

\newcounter{rem}
\newcounter{obs}
\newtheorem{remark}[rem]{Remark}
\newtheorem{observation}[obs]{Observation}
\newcommand{\mmd}{\text{MMD}}
\newcommand{\icl}{\text{ICL}}
\newcommand{\kl}{\text{KL}}
\newcommand{\rom}[1]{\text{\MakeUppercase{\romannumeral #1}}}

\makeatletter

\def\env@sqcases{%
	\let\@ifnextchar\new@ifnextchar
	\left\lbrack
	\def\arraystretch{1.2}%
	\array{@{}l@{\quad}l@{}}%
}
\makeatother

\title{Learning Invariant Representations using Inverse Contrastive Loss}
\author {
    Aditya Kumar Akash\textsuperscript{\rm 1} $\quad$
    Vishnu Suresh Lokhande\textsuperscript{\rm 1}  $\quad$
    Sathya N. Ravi\textsuperscript{\rm 2}  $\quad$
    Vikas Singh\textsuperscript{\rm 1} \\
}
\affiliations {
    \textsuperscript{\rm 1} University of Wisconsin-Madison $\quad$
    \textsuperscript{\rm 2} University of Illinois at Chicago \\
    aakash@wisc.edu, lokhande@cs.wisc.edu, sathya@uic.edu, vsingh@biostat.wisc.edu
}
\begin{document}

\maketitle

\begin{abstract}
  Learning invariant representations is a critical first step in a number of machine learning tasks. A common approach
  corresponds to the so-called information bottleneck principle in which an application dependent function of mutual information is carefully chosen and optimized. Unfortunately, in practice, these functions are not suitable for optimization purposes since these losses are agnostic of the metric structure of the parameters of the model. We introduce a class of losses for learning representations that are invariant to some extraneous variable of interest by inverting the class of contrastive losses, i.e., inverse contrastive loss (ICL). We show that if the extraneous variable is binary, then optimizing ICL is equivalent to optimizing a regularized MMD divergence. More generally, we also show that if we are provided a metric on the sample space, our formulation of ICL can be decomposed into a sum of convex functions of the given distance metric. Our experimental results indicate that models obtained by optimizing ICL achieve significantly better invariance to the extraneous variable for a fixed desired level of accuracy. In a variety of experimental settings, we show applicability of ICL for learning invariant representations for both continuous and discrete extraneous variables. {The project page with code is available at \textit{https://github.com/adityakumarakash/ICL}}
\end{abstract}

\section{Introduction}
\label{section:introduction}
%Learning invariant representations of data 
%enables 
Removing or controlling for the influence of certain 
observed or unobserved extraneous variables, that may have
an unintended effect on a learning task, is often a critical
step in model estimation \cite{xie2017controllable}. 
Often, we want to
explicitly control for their influence on the response variable, 
and estimate model coefficients that
are, roughly speaking, immune to one or more confounding
factors \cite{wasserman2013all}. 
These tasks involve
%to the need to
understanding invariance properties of data representations
and/or parameters of the model we wish to learn. 
While mechanisms to control for
extraneous variables are not strictly necessary in typical supervised 
learning tasks, where one focuses on predictive accuracy, 
over the last few years, many results have indicated 
how it can be quite useful \cite{lokhande2020fairalm}. 
For instance, 
controlling the influence of a protected attribute such as race or gender on a
response variable such as credit worthiness 
enables the design of
fair machine learning models \cite{donini2018empirical}.
Invariance is also relevant in domain adaptation 
when analyzing data from multiple sources or sites. 
Representations that are invariant
to the categorical variable (e.g., which identifies the site) leads to models 
that are more immune (or less biased) to 
{site-specific artifacts %\cite{zhou-singh-johnson-wahba-pnas-2018}}. 
\cite{Zhou1481}}.
While invariance
is prominent in a number of other settings \cite{baktashmotlagh2013unsupervised, moyer2019scanner},
we will focus on
learning representations that are minimally
informative of such extraneous variables
yet preserve enough information to reliably predict
the response/target variable or label.  

{\bf Related works.} Classical regression analysis techniques for 
handling extraneous variables based on residual scores and ANOVA \cite{girden1992anova} are not easily 
applicable for deep neural networks. 
%Since the analysis of the residuals becomes trickier with deep neural networks. As a result,
Instead, one approaches the question in one of two ways.
%In the adversarial approach  that is common in
%the literature,
A common approach is to use a standalone
adversarial module \cite{xie2017controllable} tasked with
using the latent representations of the data to predict the
extraneous variable whose influence
(on the representations) we wish to remove.
If the adversary succeeds, we have not yet fully controlled for
the extraneous variable and so, the representations
%(rather, the embedding function)
must be modified.
This necessitates the design of an adversary tailored to 
the form of the downstream 
%learning 
task. Further, the
evaluation of sample complexity, convergence behavior of the training procedure,
and the degree to which the representations remain invariant when the datasets
are scaled or if an additional confound must be controlled for, require careful
treatment and remain an active area of research \cite{jaiswal2019unified} \cite{NIPS2018_7756}. 
An alternative strategy is to ask for
{statistical independence 
of the latent representations learned by the network 
%with respect to
and
the extraneous variable.}
For example, one may approximately measure mutual information \cite{cover1999elements} between
the latent representations
and the extraneous variable \citet{moyer2018invariant}.
This idea as well as
the use of alternative distance and divergence measures is popular
\cite{li2014learning, Louizos2016TheVF}, and in most cases, perfectly
models the innate requirements of the task.
In practice, however, 
{their viability 
%for an application 
depends} on a variety of computational
and implementation considerations, where design/approximation choices
may frequently lead to representations where a modest adversary
can successfully recover information about 
the extraneous variable fairly reliably.

\begin{figure}[!t]
	\centering
	\includegraphics[width=0.9\linewidth]{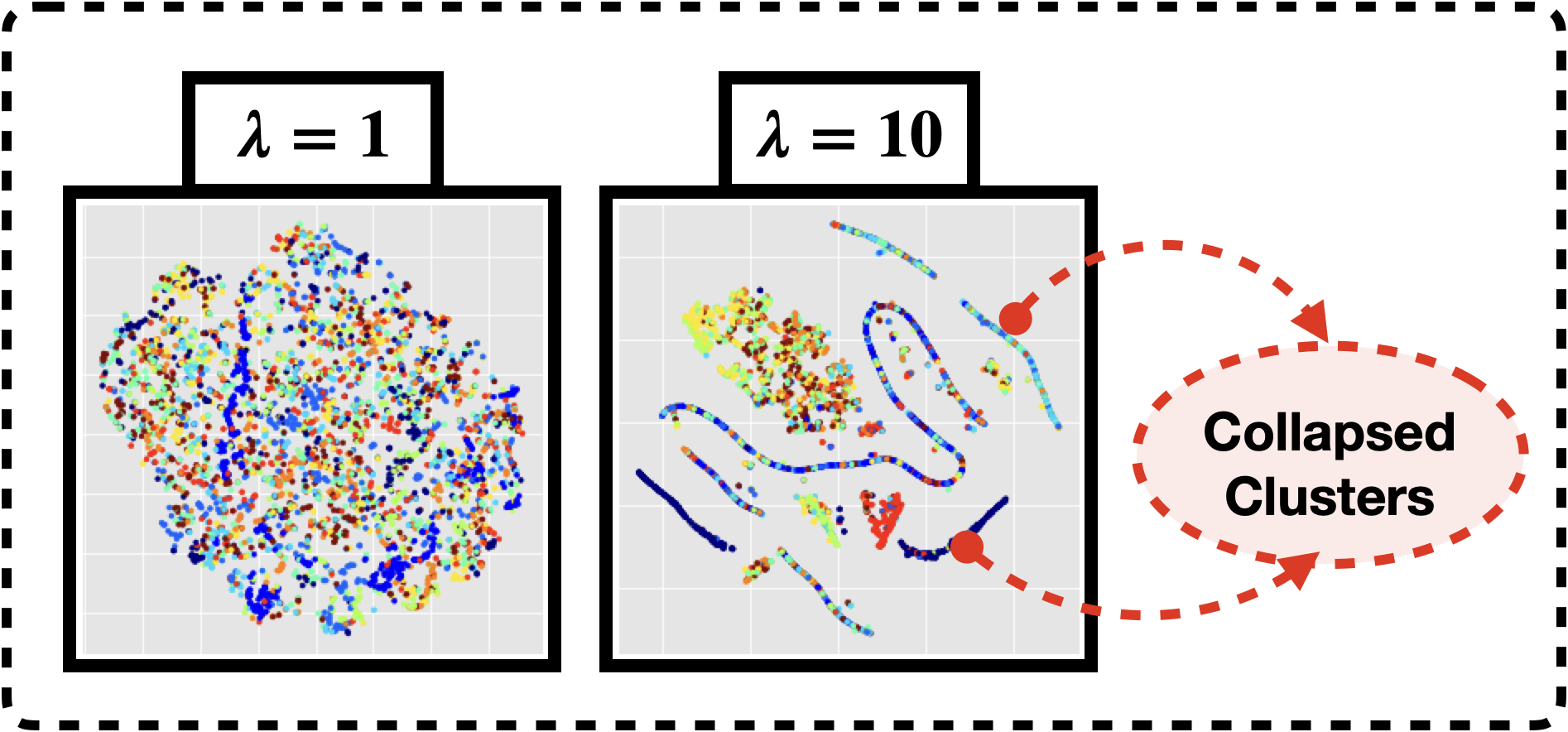}
	\caption{t-SNE plots for MNIST style experiment where the digit label 
		is the extraneous variable $c$. 
		For existing compression regularizers, increasing the regularization weight $\lambda$ results in the collapse of latent space as indicated by the plot.}
	\label{fig:penalty_property1}
\end{figure}

{\bf Moving from theory to practice.}
To operationalize the statistical independence criterion described above,
a sensible modeling choice is to use mutual information \cite{moyer2018invariant} and then
choose a good approximation. We describe this setting briefly to 
identify some practical issues that affect the overall behavior: 
%
%consider the setting where statistical independence is imposed
%via mutual information, as in,
instead of minimizing the mutual information $I(z,c)$, 
where
$z$ denotes the latent representation of the data $x$ and
$c$ denotes the extraneous variable, 
we may minimize a suitable upper bound instead as shown below
%is minimized as a measure
%of invariance
%. The proposed regularizer 
{
	\begin{align}
	I(z, c) \le \underbrace{\mathbb{E}[KL[q(z|x) || q(z)]]}_{\text{$a)$ Compression}} - \underbrace{\mathbb{E}[\log(p(x|z,c))]}_{\text{$b)$ Reconstruction}}
	\label{eq:mi_bound}
	\end{align}}
{
  The bound in  \eqref{eq:mi_bound} considers contributions from two terms:
  \begin{inparaenum}[\bfseries (a)]
  \item one that \textit{compresses} $x$ into $z$ via an encoder modeled using 
    conditional likelihood $q(z|x)$ (whose marginal is $q(z)$), and
  \item the second which
    \textit{reconstructs} $x$ from $z$ and $c$ via a decoder
    \end{inparaenum}
	%modelled using the variational distribution
    $p(x|z,c)$.
	%The relative importance of the two terms is determined using a weight parameter $\lambda$.
	Clearly, if $c$ is available for free during decoding,
	there is no reason for the model not to aggressively compress $x$, 
	while keeping just enough information content to reliably reconstruct
	it during decoding.
	When both terms function as intended, the balance will lead to
	representations that are invariant to $c$, as desired.

	Let us temporarily set aside the reconstruction term and
	evaluate the compression term in \eqref{eq:mi_bound} which 
	ideally will remove from $z$ the information regarding $c$. 
	It is used as an invariance regularizer and controlled using a weight parameter $\lambda$ 
	as follows
	\begin{align}
	\underbrace{\lambda}_{\text{Weight}} \underbrace{\mathbb{E}[KL[q(z|x) || q(z)]]}_{\text{ Compression Regularizer}}
	\label{eq:compress}
	\end{align}
	To minimize \eqref{eq:compress} in a computationally tractable way, 
	one may model $q(z|x)$ 
	as a Gaussian which allows \eqref{eq:compress} to be approximated using 
	pairwise distances $KL(q(z|x_i)$, $q(z|x_j))$, where $x_i,x_j$ are different input samples. 
	With this assumption, $KL(q(z|x_i), q(z|x_j))$ admits a closed form and roughly translates to 
	the difference between the means of these two Gaussians scaled by the covariance \cite{wasserman2013all}.
	For a reasonable weighting $\lambda$, we obtain some invariance to $c$ in $z$
	but an adversary can still recover $c$ from $z$.  
	Increasing $\lambda$ -- to improve the invariance behavior --
	leads to the means of the conditionals coming closer to each other.
	Since the mean of the conditional $q(z|x)$ is used as the
	encoded representation of $x$, this also brings the representations closer together.
	Note that this compression of the means is agnostic of the extraneous variable $c$.
	In practice, this leads to a collapse of the latent space and
	formation of clusters when the strength of the regularizer is increased, making it easier
	for an adversary to recover $c$ from $z$. }
%  \begin{figure}[!b]
%  	\centering
%  	\includegraphics[width=0.99\linewidth]{images/list_obs_des.png}
%  	\caption{\label{fig:penalty_property} The latent space generated by existing regularizers gives some invariance with moderate weights $\lambda$, but form clusters for large $\lambda$ values. The desired behaviour is to spread intraclass samples and mix interclass samples giving rise to high invariance. The proposed regularizer ICL intuitively captures this notion.}
%  \end{figure}
\begin{figure}[!b]
	\centering
	\includegraphics[width=0.99\linewidth]{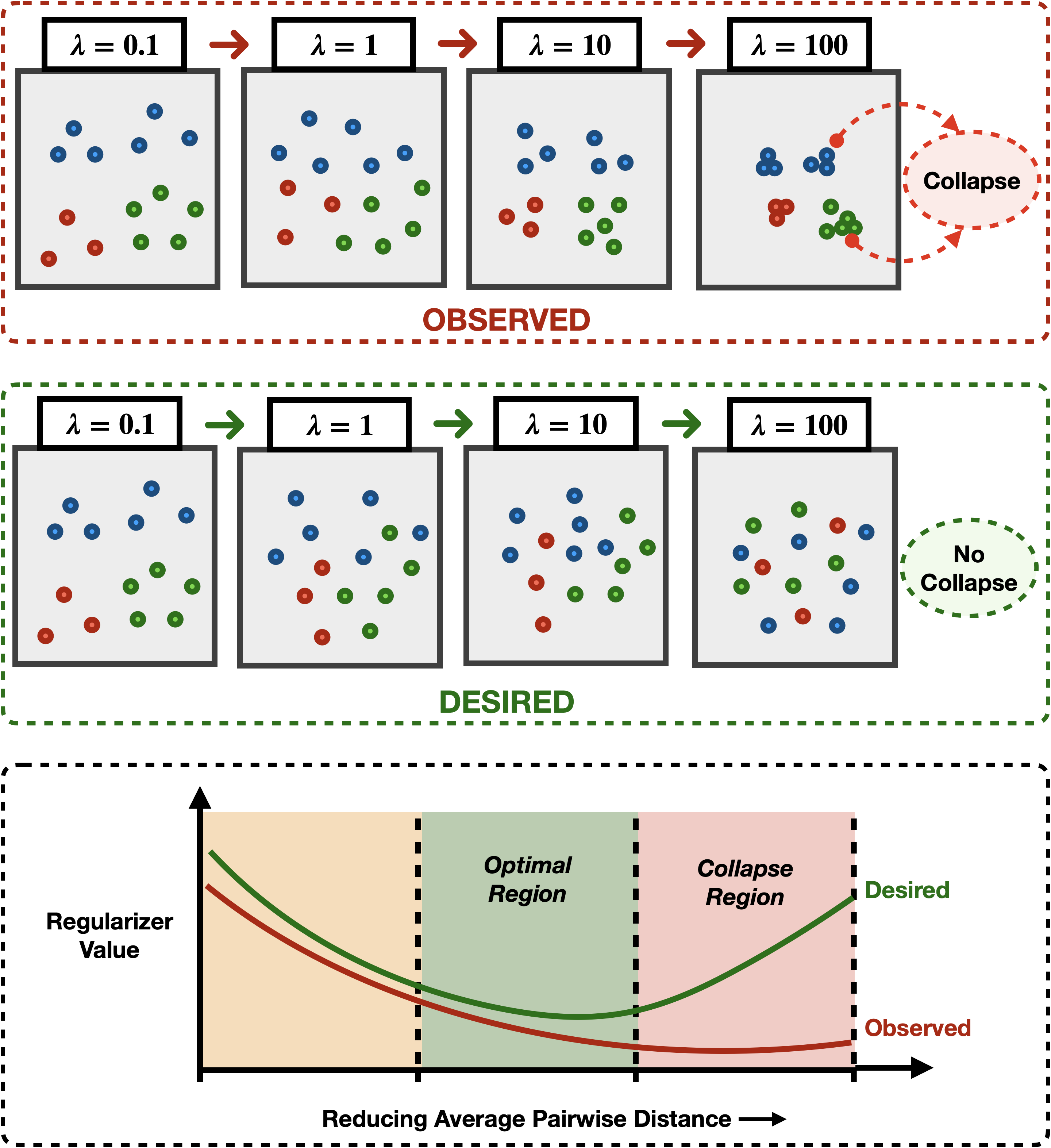}
	\caption{\label{fig:penalty_property} \textbf{(Top):} Representations generated by existing regularizers have some invariance for moderate weights $\lambda$, but form clusters for large $\lambda$ values. \textbf{(Center):} The desired behavior is to spread intraclass samples and mix interclass samples giving rise to high invariance. The proposed ICL regularizer intuitively captures this notion. \textbf{(Bottom):} Existing compression regularizers are observed to let average 
		distance between samples decrease and discourage cluster formation. A desired regularizer would assign high penalty in the collapse region and prevent clustering.}
\end{figure}

{\bf An example.} We illustrate the above behavior experimentally using the setup of \cite{moyer2018invariant}
for unsupervised representation learning in MNIST in Figure  \ref{fig:penalty_property1}.
We wish to learn representations which are only informative of the style of digits
but uninformative of the digit label. 
We gradually increase the strength of the compression term, via the weight parameter $\lambda$, and evaluate its effect.  
Since images of the same digits are similar to begin with, they map to representations
which are in close proximity.
This means that the latent space already has a rough grouping
of representations based on the digits.
A modest increase of the compression strength causes the inter group distances to decrease. This makes it 
more difficult to distinguish one group from another and can be seen as improving invariance --
but not yet enough that an adversary cannot recover the digit label from the representation. 
However, when the regularizer is increased further,
we observe that the latent representations 
start to form smaller clusters associated with the variable $c$ 
or collapse completely -- degrading invariance -- in fact, 
making it easier for the adversary to identify the digit class. 
%Figure \ref{} shows this observation. 
%This clustering behavior makes it easier for the adversary to idenify the extraneous class. 
%Thus after a threshold of strength the invariance actually degrades. 
The foregoing behavior {(see Figure \ref{fig:penalty_property})} is
not an artifact of approximation choices. 
Consider a latent variable $z$ and a binary variable $c \in \{0,1\}$ we wish to control
for. Here, we are concerned with the conditional distributions $p(z|c=0)$ and $p(z|c=1)$.
Let us assume that we use a divergence $\mathcal{D}$, and
statistical independence between $z$ and $c$ is imposed
by a soft-version of the constraint $\mathcal{D}(p(z|c=0), p(z|c=1)) = 0$.
{For each value of $c$, if the latent space has clusters to begin with 
	and one optimizes
	{\em both} $p(z|c=0)$ and $p(z|c=1)$ together, with no mechanism to spread/inter-mix the 
	representations, the latent space may remain clustered with respect to $c$ 
	when we increase the weight of the invariance term. The above issue has less to do with
	how the distributional overlap is measured and can instead be attributed to
	not discouraging the formation of clusters. }
It seems that an explicit use of the extraneous variable during the encoding step may provide
an effective workaround.

The basic intuition expressed above,
is that $x$'s that pertain to different values of $c$'s should map to representations
$z$'s which are ``mixed'' yet contain enough information to keep the reconstruction error low.
At the same time,
representations for a specific value of $c$ should be spread out, and
not locally collapse 
to a point even when the weight parameter $\lambda$ is increased. 

{The \textbf{main contributions }of this paper include $(a)$ We propose {\em Inverse} Contrastive Loss (ICL) for learning invariant representations inspired from the class of contrastive losses \cite{lecun2005loss}. Our proposed loss is  computationally efficient as it does not require specialized solvers or additional training through adversarial modules. $(b)$ We interpret ICL by drawing a relation with the well studied Maximum Mean Discrepancy (MMD) as well as energy functionals used in dynamical systems analysis. $(c)$ We demonstrate that ICL provides invariant representations for not only discrete extraneous variables but also continuous ones.}

\section{From Contrastive Models to Inverse Contrastive Representation Learners}
We will now briefly review concepts from the recently proposed framework of {\em Contrastive Loss (CL}) functions. We will denote our input data using tuple of random variables  $(x,x^-) \in \R^{d_1}\times \R^{d_2}$ where $x^-$ is a {\em negative} sample, that is, if $x$ can semantically be classified as $y$, then $x^-$ is closer  to a different class $y^-\neq y$. As usual, in unsupervised learning, $y,y^-$ are not available during training.   Let $z$ (and similarly$,z^- $) denote the latent representation of $x$ that may be obtained using a feature extraction
scheme like ResNet, DenseNet or others \cite{he2016identity}.
Finally, a CL function is defined by $ \ell \left(z^T\left(z^+-z^-\right)\right)$ where $z^+$ is the representation of a sample from the same class as $x$ and $\ell$ can be  any classification loss function such as hinge, softmax etc., see Definition 2.3 in \citep{saunshi2019theoretical}. In essence, the definition of CL function captures the simple notion of {\bf contrastiveness}  that {\em semantically} similar points should have  {\em geometrically} similar representations \citep{1640964}. To see this, assume that $\ell$ is the logistic loss, then it is easy to see that $ \ell\left(z^T\left(z^+-z^-\right)\right)$ is small for a high intraclass similarity $z^Tz^+$ and a low interclass similarity $z^Tz^-$.  We say that a model is {\em contrastive} if it satisfies the contrastiveness property. We will now list some basic mathematical notations that we will use throughout the rest of the paper.  %This similarity is often expressed as input belonging to the same class or feature being invariant under some transformation. 

{\bf Basic Notations.} For any pair of random variables $(x_1,x_2)$, we will use $p(x_1,x_2), p(x_1|x_2)$ to denote the joint and conditional distribution respectively. $\updelta(x)$ represents the dirac delta measure at $x\in\R^d$, and the indicator function $\mathbbm{1}(\cdot)$ evaluates to $1$  if the argument is true, and $0$ otherwise. For a positive definite kernel $k(x,y)$, MMD divergence \citep{gretton2008kernel} between distributions $p, q$ is defined as, 
\begin{align}
\mmd_{k} = \underset{\substack{x \sim p \\ x' \sim p}}{\mathbb{E}}k(x,x') + \underset{\substack{y \sim q \\ y' \sim q}}{\mathbb{E}} k(y,y') - 2\underset{\substack{x \sim p \\ y \sim q}}{\mathbb{E}} k(x,y)
\label{eq:mmd}
\end{align}
For $z,z'\in \R^d$, we will use $d(z,z')$ to be the Euclidean norm $\|z-z'\|_2$ unless otherwise stated, and
$\mathcal{N}_{\delta}(z)$ denotes the Euclidean ball of radius $\delta$ centered at $z$.  For a subset $X\subseteq \R^d$, we will use $\mathcal{P}(X)$ to denote the space of probability distributions over $X$.

\subsection{How to invert a CL function to learn  invariant representations?}
In this section, we will define our Inverse Contrastive Loss (ICL) that can be used to learn  representations that are  {\em invariant}  to an extraneous (random) variable $c$.  At a high level, our procedure consists of the following two steps: 
\begin{enumerate}
	\item {\bf Formal Inversion (FI):} {invert the contrastiveness property to reflect low intraclass and high interclass similarity by switching the role of  $z^Tz^-$ and $z^Tz^+$ via sign flip;}
	\item {\bf Addition of Weighted Neighborhood Kinks (AWNK): } {apply an increasing function on interclass similarity $z^Tz^-$ and a decreasing function on the intraclass similarity $z^Tz^+$}
\end{enumerate}
While the two step procedure mentioned above implicitly defines an Inverse CL (ICL) function, note that it
is well defined as long as the CL function is. Before we present a precise
definition of ICL, it is meaningful to see why FI+AWNK can improve invariance to an extraneous variable.
%Step 1. Take contrastive loss with loss functions for similar and dissimilar pairs 2. Swap the two losses, use the similar pair loss for dissimilar extraneous values and corresponding operation for the similar extraneous variable.      We use a different loss function for the dissimilarity in ICL (push). Where alpha controls the importance assigned to pushing (similar to alpha in ICL loss). Beta controls the steepness of the slope, a new thing which is not there in the original loss. We introduce a class of invariance losses by inverting the class of constrastive losses to address the challenges. We call these losses Inverse Contrastive Loss (ICL).  In the context of learning representations invariant to some extraneous variable, it is very difficult to express this kind of invariance using an explicit transformations and features which have similar values for extraneous variable are often similar. 

{\bf  Sufficiency of FI+AWNK.} As discussed earlier in Section \ref{section:introduction},
for learning invariant representations, it is desirable that features with similar $c$  be spread apart in the latent space while the features with dissimilar $c$ be closer to each other.
FI explicitly formalizes the idea that invariance should be
better for
high interclass similarity $z^Tz^-$ and a low intraclass similarity $z^Tz^+$.  AWNK can be thought of as a disentanglement step that allows us to handle interclass and intraclass similarities appropriately.
{
	The interclass similarity  $z^Tz^-$ is expressed with a quadratic function similar to \cite{1640964}. For intraclass similarity $z^Tz^+$,  \cite{1640964} suggests using a clipped quadratic function which is inefficient for gradient based methods because the gradient in the clipped region of the function is always zero. In contrast, %\cite{1640964},  
	we propose to use exponential loss which provides non-zero gradient values. While other alternative functions are  applicable here, we will see shortly in Section \ref{section:spring_force} that the exponential loss provides a means to draw an interesting connection between ICL and well-studied and mature ideas like MMD divergence. To sum up, AWNK's role is to prevent the intraclass representations from locally collapsing {\bf even} for a wide range of values of the regularization parameter. }

\subsection{ICL -- A Probabilistic Definition}
From now on, we will use distance/metric $d(z,z')$ to measure similarity -- closer points are similar.
Intuitively this can be expressed by saying that on average features which share similar values for extraneous variables have representations that are further from each other, while features that have dissimilar values for extraneous variables have closer representation. We operationalize this intuition by inverting the class of contrastive losses.

\begin{definition}\label{def:icl_joint}
	Let $p(z,c)$ be the joint distribution for representation variable $z\in {Z}$ and extraneous variable $c \in \mathcal{C}$. Let $d_Z(z,z')$ be the distance metric on ${Z}$ and $\mathcal{N}_\delta(c)$ denote the $\delta$-neighbourhood centered at $c$. For $s(z,z') = d^2_Z(z,z')$ and $f(z,z')=\exp{(\alpha - \beta d_Z(z,z'))}$, $\beta > 0$, we define $\icl_{\alpha, \beta}^{\delta}(z,c):\mathcal{P}({Z}\times \mathcal{C}) \mapsto \mathbb{R}_{+}$ as
	\begin{align}
	\icl_{\alpha, \beta}^{\delta}(z, c) = \underset{\substack{(z,c) \sim p(z,c) \\ (z',c') \sim p(z', c')} }{\mathbb{E}}&\bigg[\mathbbm{1}(c' \in \mathcal{N}_\delta(c)) f(z,z') \; +\\
	& \mathbbm{1}(c' \notin \mathcal{N}_\delta(c)) s(z,z')\bigg]	
	\label{eq:icl_joint}
	\end{align} 
\end{definition}

In Definition \ref{def:icl_joint}, $\mathcal{N}_\delta(c)$ encodes the similarity aspect of extraneous variables using its (underlying) geometry. A simple calculation shows that ICL functions in \eqref{eq:icl_joint} immediately possess two (desirable) geometrical properties by definition:
\begin{inparaenum}[\bfseries (a)]
\item whenever samples have similar extraneous value, our  loss function is specified by $f(z,z')$ -- a {\em decreasing} function of $d_Z(z,z')$; and \item for samples with dissimilar extraneous value, the loss is specified by $s(z,z')$ -- an {{\em increasing} function of $d_Z(z,z')$}.
\end{inparaenum}
For the remainder of the paper, we will hide the AWNK parameters $\alpha, \beta$ and radius $\delta$ in  ICL functions \eqref{eq:icl_joint} whenever appropriate.

\begin{figure*}[!t]
	\centering
	\includegraphics[width=0.99\linewidth]{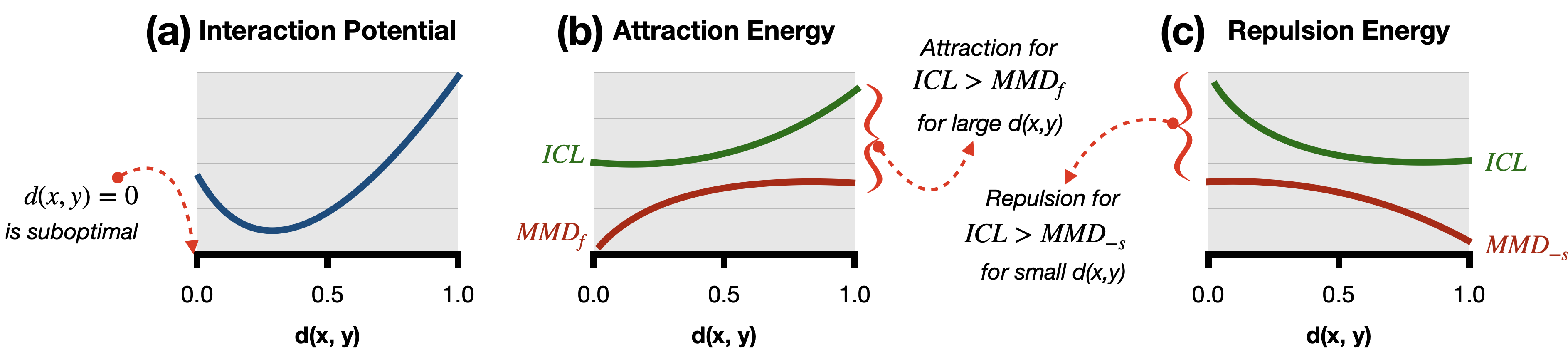}
	\caption{ \label{fig:icl_dynamics} (a) We plot the interaction potential for the functional $R_w$. The functional $R_w$ prevents the collapse of representation space by shifting the minima away from the trivial solution $d(x,y)=0$. (b) We compare the attraction energy between $\icl$ and $\mmd_{f}$. The attraction for $\icl$ is larger than for $\mmd_{f}$ when the particles are farther apart. (c) We plot the repulsion energy of  $\icl$  and $\mmd_{-s}$. The repulsion for $\icl$ is larger than $\mmd_{-s}$  when the particles are in close neighborhood.}% and as a result pushing them apart.}% This pushes the particles strongly when they are in close neighbourhood.  %The attraction between interclass particles is given by gradient of energy w.r.t. distance $d(x,y)$, $\|\nabla_ds(x,y)\|_2$  and $\|\nabla_df(x,y)\|_2$.}
\end{figure*}

\begin{remark}It turns out that optimizing ICL  is equivalent to driving a  spring system to equilibrium in which samples with similar extraneous values are connected by a push spring while samples with dissimilar extraneous values are connected by a pull spring, see \citep{1640964}. In particular, the neighborhood radius $\delta$  in our ICL functions \eqref{eq:icl_joint} determines the level of control exerted by these connections in the system -- a large $\delta$ forces the latent representations to come closer while a smaller $\delta$ drives the representation to be a bit more spread. 
	
\end{remark}

{\bf Handling Discrete Extraneous Variables using ICL.} The following Lemma states that definition of ICL function in \eqref{eq:icl_joint} is closely related to the standard MMD distance in \eqref{eq:mmd}.  \begin{lemma}[ICL is equivalent to R-MMD] \label{lem:icl-r-mmd}Assume that the extraneous variable $c$ is binary with $p(c=0)=1/2$. Then there exist a conditionally positive definite kernel $g$ and an interaction energy functional {$R_w$}  (see equation 1.1 in \citep{carrillo2012gradient}) such that the following equality holds:\begin{align}
	\icl(Z, C) = \mmd_{g}(p_0, p_1) + R_w(p_0, p_1),\label{eq:icl-r-mmd}
	\end{align}
	where $p_0$ and $p_1$ denote the conditional distributions $p(z|c=0)$ and $p(z|c=1)$ respectively.
\end{lemma}
The proof of Lemma \ref{lem:icl-r-mmd} is included in the appendix. In essence, Lemma \ref{lem:icl-r-mmd} states that if $c$ is binary, then optimizing $\icl$ is equivalent to optimizing a Regularized-MMD (R-MMD) divergence between conditional distributions $p(z|c=0)$ and $p(z|c=1)$. % If we let $p_c$ denote the conditional distribution $p(z|c)$ and assume that the classes $c\in \{0, 1\}$ are equilikely, the $ICL$ can be written as  
Recall from Section \ref{section:introduction} that $\mmd(p_0, p_1) = 0$ is a sufficient condition for statistical independence between $z$ and $c$.  
Hence, for the special case considered here, we see that $\icl$ ensures statistical independence constraint using R-MMD divergence. To see that $\mmd_g$ is a valid divergence, note that the kernel $g$ is conditionally positive definite since it is a composition of a laplacian kernel and a euclidean distance matrix. Please see appendix for details on how to generalize Lemma \ref{lem:icl-r-mmd} to multiclass setting, when $c$ is (discrete) uniformly distributed.

In practice, we are often {\em only} given access to empirical samples of $z$ and $c$. This becomes problematic for optimization purposes since  we can only evaluate the divergences {\bf approximately}  -- approximate zeroth order oracle.  In the next section, we study the finite sample optimization properties of R-MMD \eqref{eq:icl-r-mmd} using control theoretic constructions.

\subsection{Exploring the Landscape of ICL functions using Spring Forces}
\label{section:spring_force} The following observation establishes a link between the $R_w$ term in \eqref{eq:icl-r-mmd} and {\em distributional interaction energy functionals} used in analyzing dynamical systems \cite{carrillo2012gradient}.  
\begin{observation}[Significance of $R_w$]\label{rmk:sig_icl}
	{The regularizer $R_w(p, q)$ is composed of pairwise energy functional $w(x,y) \sim f(x,y)+s(x,y)$ between particles of the system \citep{1640964}.} Intuitively, when input distributions $p$ and $q$ are decision variables of an optimization problem, $\mmd_g$ admits a trivial solution, $p=q=\updelta(0)$, that is, $p$ and $q$ collapse to a single point mass. However, this trivial solution is almost surely {\em suboptimal} for $R_w$ (see Figure \ref{fig:icl_dynamics}a.), thus decreasing the chances of such a collapse. 
	Indeed, since $R_w$ forces representations to stay apart even when the regularizer weight is arbitrarily increased, which suggests that $R_w$ may be reasonable for learning invariant representations. 	\label{remark-2}
\end{observation} 

Plugging in the definition of $R_w$ ({see appendix}) in equation \eqref{eq:icl-r-mmd} and rearranging, we have that,
\begin{align}
\icl(z, c) 
= \overbrace{\underset{\substack{x \sim p_0\\
			x' \sim p_0}}{\mathbb{E}} \mbox{\Large$\frac{f(x,x')}{4}$}+\underset{\substack{y \sim p_1\\y'\sim p_1}}{\mathbb{E}} \mbox{\Large$\frac{f(y,y')}{4}$}    }^{\text{Repulsion}}+\overbrace{\underset{\substack{x \sim p_0 \\ y \sim p_1}}{\mathbb{E}}   \mbox{\Large$\frac{s(x,y)}{2}$}   }^{\text{Attraction}}
\label{eq:icl-simple}
\end{align}	
Intuitively,  \eqref{eq:icl-simple} shows that ICL can be decomposed into two terms: \begin{enumerate*}\item {\bf Attraction} $s(\cdot,\cdot)$ between
	interclass particles; and \item {\bf Repulsion} $f(\cdot,\cdot)$ between intraclass particles. \end{enumerate*} That is, ICL can be interpreted as modeling interclass and intraclass connection between particles (representations)  using two types of springs $f,s$. Indeed,  a similar decomposition is also possible for MMD by setting $s=-k,f=k$. For optimization purposes, our choice of $f$ and $s$ in R-MMD immediately yields {\bf two}  crucial benefits that is absent in MMD:

{\bf Benefit 1 -- ICL is well suited for First Order Methods.}   By definition, gradient of spring energy with respect to the distance $d(x,y)$ is  the sum of attraction and repulsion connecting two particles. 
ICL and MMD$_{f}$ differ in the attraction spring between interclass samples. 
When the distance between samples $d(x,y)$ is large, the attraction under ICL given by $\|\nabla_ds(x,y)\|_2$ is larger than the attraction under $\mmd_{f}$ given by $\|\nabla_df(x,y)\|_2$ ({Figure \ref{fig:icl_dynamics}b}). 
Furthermore attraction $\|\nabla_ds(x,y)\|_2$ increases with $d(x,y)$ while $\|\nabla_df(x,y)\|_2$ deceases.
Hence, while using first order methods like gradient descent, 
farther particles come closer {\em faster} while using ICL.

%Referring back to our intuition of intermixing interclass samples for invariance, 
%ICL in comparison to MMD$_{f/4}$ uses a term which is more effective for this.
{\bf Benefit 2 -- ICL prevents particles from collapsing.} In the context of learning invariant representations, ICL and MMD$_{-s}$ differ in repulse-only springs between intraclass samples. {{}{}For ICL, the repulsive forces $\|\nabla_df(x,y)\|_2$ between samples increases as the particles come close together while for $\mmd_{-s}$ the force $\|\nabla_ds(x,y)\|_2$ decreases ({Figure \ref{fig:icl_dynamics}c}). Hence, whenever gradient based methods are used for training, ICL may be beneficial since the intraclass particles are pushed apart {\em strongly} when they are in the same neighborhood, as desired.  }

{\bf ICL Optimization provides adversarially invariant representations.}  %Lemma \ref{lem:icl-r-mmd} already motivates the use of ICL for imposing invariance through lens of statistical independence and the above discussion outlines its useful properties as a loss function.
%Next we provide further intuition on how using ICL 	for adversarial invariance is sensible.
It turns out that the above two benefits can be used to prove that models obtained by optimizing ICL derived loss can confuse adversaries. Formally, consider an adversary $b$ that uses representation $z$ to predict a continuous extraneous variable $c$. %continuous extraneous variable $c$ and an adversary $b$, 
We will use the mean squared error (MSE) $\mathbb{E}_z[(b(z)-c)^2]$ to measure invariance, that is, a high value of MSE implies high  invariance (desired). 
The following Lemma provides a lower bound on the MSE as a function of ICL under standard assumptions on $b$. \begin{lemma}
	Assume that the extraneous variable $c$ is continuous and $b$ is $L$-lipschitz, and let $\rho = P_{c,c'}(|c-c'| > \delta)$. Then there exists $\alpha$, and $\epsilon < \delta^2\rho^2/L^2$ such that for $\icl^{\delta}_{\alpha, \beta}(z,c) < \epsilon$, the MSE of adversary $b$ is lower bounded i.e, $ \mathbb{E}_z[(b(z)-c)^2] \ge (\delta \rho - L\sqrt{\epsilon})^2/4$. \label{lemma-mse}
\end{lemma}
The proof of Lemma \ref{lemma-mse} is included in the appendix. Basically, Lemma \ref{lemma-mse} states that if ICL is made sufficiently small, then no Lipschitz adversary can have an arbitrarily small MSE {as expected}.  {  We will now demonstrate the utility of Lemma \ref{lemma-mse} for analyzing datasets used in real world applications.
}

\section{Applications of Inverse Contrastive Loss }\label{sec:apps}
Many representation learning schemes are built on
Variational Auto-Encoder (VAE) based models \citep{kingma2013autoencoding}.
Recently \citep{Cemgil2020Adversarially} showed that
one effective mechanism to 
improve adversarial robustness of representations obtained using VAE
is via data augmentation: creating ``fictive''
data points. 
This can be thought of as providing
invariance w.r.t. adversarial perturbations. 
However, obtaining such perturbations
%one or
%more extraneous variables
might not always be possible.
While rotations, flips and crops work
for natural images,
this is problematic for brain imaging data where either
a cropped brain or an image-flip that switches
the asymmetrical relationship between the two hemispheres
is meaningless. Applying a deformation
to generate an augmented sample is defensible,
but requires a great deal of care and user involvement.
Similarly, deploying augmentation strategies for
electronic health records (EHR) or audio data is
not straightforward. 
Section \ref{lemma-mse} provides us the necessary
guidance to explore the use of
ICL regularizer for VAE based representation learners. 

{\bf Setup.} We use the setup based on Conditional VAE and
Variational Information Bottleneck (VIB) \cite{45903} for learning invariant representations 
in unsupervised and supervised setting respectively.
These frameworks have been considered in the context of a
mutual information based regularizer by \cite{moyer2018invariant}.
Briefly, in the unsupervised setting one learns representations $z$ using 
an encoder $q(z|x)$, that maps data $x$ to conditional distribution $q(z|x)$, and 
a decoder $p(x|z, c)$ that reconstructs $x$ from $z$ and $c$.
Gaussian reparameterization trick allows the encoder $q(z|x)$ to be written as $\mathcal{N}(\mu=h(x), \sigma(x))$, where $h(x)$ is the representation learner of interest.
We augment this setup with ICL regularizer and propose optimizing the following objective, 
%\begin{align}
%\min_{p, q}  \mathbb{E}_{x,c}\big[\mathbb{E}_{z}[-\log p(x|z,c)] & + \beta \kl[q(z|x) || p(z)]\big] \\
%&+ \lambda \icl(z, c)
%\label{eq:vae_obj}
%\end{align}
\begin{align}
\min_{p, q}  \underset{x,c}{\mathbb{E}}\big[\underset{z}{\mathbb{E}}[-\log p(x|z,c)] & + \beta \kl[q(z|x) || p(z)]\big] \\
&+ \lambda \; \icl(z, c)
\label{eq:vae_obj}
\end{align}
where $p(z)$ is standard isotropic Gaussian prior.

For the supervised setting of predicting $y$ from $x$
we augment the VIB framework from \cite{45903} with ICL regularizer and propose optimizing the following objective
%\begin{align}
%\min_{p,q} \mathbb{E}_{x,c}\big[\mathbb{E}_{z,y}[-\log p(y|z)] &+ \beta \mathbb{E}_z[-\log p(x|z,c)]\big] \\
%&+ \lambda \icl(z, c)
%\label{eq:vib_obj}
%\end{align}
\begin{align}
\min_{p,q} -\underset{\substack{x\\c}}{\mathbb{E}}\bigg[\underset{\substack{z\\y}}{\mathbb{E}}\log p(y|z) + \beta \underset{z}{\mathbb{E}}\log p(x|z,c)\bigg] + \lambda \; \icl 
\label{eq:vib_obj}
\end{align}
where $p(y|z)$ is the learned prediction model.

Next, we show ICL's wide applicability by using it with discriminative encoders
that are not based on VAE.
Consider the task of predicting $y$ from $x$ in presence of 
extraneous $c$. To learn representations uninformative of $c$, 
the task is broken down into
learning an encoder $h:x\mapsto z$ and a
predicter $f:z\mapsto y$. We add the ICL regularizer to the loss objective $\ell$ 
and propose optimizing
\begin{equation}
\min_{f, h} \underset{x,y}{\mathbb{E}}[\ell(f(h(x), y))] + \lambda \icl(h(x), c)
\label{eq:sup_obj}
\end{equation}
$h$, $f$ are generally parameterized using deep networks.

{\bf Baselines.}
As discussed in Section \ref{section:introduction} invariance can be enforced using statistical independence or using adversarial modules.
Our proposed ICL loss is compared with the following frameworks from both these categories: 
\begin{enumerate*}
	\item[(a)] Unregularized model,
	\item[(b)] MI regularizer \citep{moyer2018invariant},
	\item[(c)] OT regularizer, where KL term in (b) is replaced with
	Wasserstein distance, 
	\item[(d)] MMD$_{-s}$ (Section \ref{section:spring_force}), 
	\item[(e)] MMD$_f$ (Section \ref{section:spring_force}), 
	based on MMD \cite{li2014learning} and 
	\item[(f)] CAI, Controllable invariance through adversarial feature learning \cite{xie2017controllable}, 
	\item[(g)] UAI, Unsupervised Adversarial Invariance \cite{NIPS2018_7756}.
\end{enumerate*} 

{\bf Quantifying invariance.}  We follow \citep{xie2017controllable} and train a three layered
FC network as an adversary 
to predict the extraneous variable $c$ from latent representations $z$. 
We report the accuracy of this adversary for discrete $c$ 
and MSE for continuous $c$ as 
{the adversarial invariance measure (A)}.

We evaluate the frameworks in terms of task accuracy/reconstruction error
and adversarial invariance on an unseen test set. 
The hyperparameter selection is done on a validation split 
such that best adversarial invariance is achieved for task accuracy 
within $5\%$ of unregularized model for supervised tasks and 
reconstruction MSE within $5$ points of unregularized model 
for unsupervised tasks.  Mean and standard deviation are reported 
on ten runs, except when mentioned otherwise or 
quoting results from previous work.
{We use Adam optimizer for model training. 
More details on training}
and hyperparameters are provided in the appendix.
Next, we present our results grouped by the nature of model
(generative/discriminative) and the dataset.

\begin{figure}
	\centering
	\includegraphics[width=0.9\linewidth]{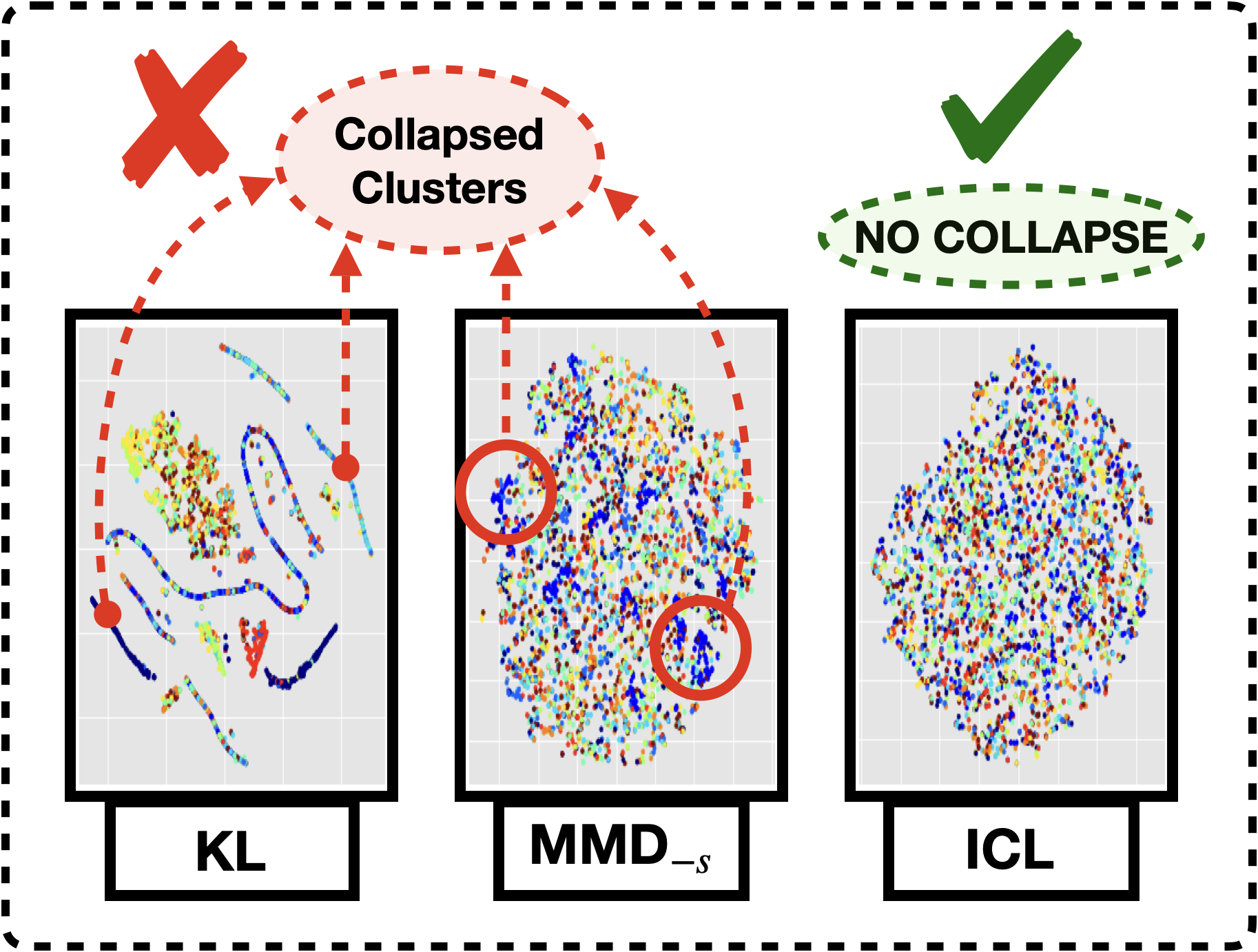}
	\caption{ We plot t-SNE for latent representations of KL, $\mmd_{-s}$ and $\icl$ for MNIST style experiment. Collapsed clusters are observed in the plots of KL, $\mmd_{-s}$, whereas $\icl$ generates a uniform latent space favoring invariance. }
	\label{fig:tsne_mnist}
	
\end{figure}

\subsection{Generative Model Families}
\begin{table*}[t] 
	\footnotesize
	\captionsetup{justification=centering} 
	\setlength\tabcolsep{0pt} % make LaTeX figure out width of inter-column spaces
	\caption*{\textbf{R:} Reconstruction Error, \textbf{P:} Prediction Accuracy, \textbf{A:} Adversarial Invariance Measure\\
		$\uparrow$: Higher Value is preferred, $\downarrow$: Lower Value is preferred}
	\vspace{-1mm}
	\begin{tabular*}{\textwidth}{l @{\extracolsep{\fill}}
			*{10}{S[table-format=1.4]}}
		\midrule\midrule
		& \multicolumn{2}{c}{MNIST} & \multicolumn{2}{c}{Adult} & \multicolumn{2}{c}{German}& \multicolumn{2}{c}{MNIST-ROT}& \multicolumn{2}{c}{ADNI} \\
		\cmidrule{2-3} \cmidrule{4-5} \cmidrule{6-7} \cmidrule{8-9} \cmidrule{10-11}
		& {R $\downarrow$} & {A $\downarrow$} & {P $\uparrow$} & {A $\downarrow$} & {P $\uparrow$} & {A $\downarrow$} & {P $\uparrow$} & {A $\downarrow$}  & {P $\uparrow$} & {A $\downarrow$}\\ 
		\midrule\midrule
		Unregularized & {$12.1 \pm 0.5$} &  {$46 \pm 4$}  &  {$84 \pm 0$}  &  {$84 \pm 0$}  &  {$73 \pm 2$}  &  {$78 \pm 2$}  &  {$96 \pm 0$}  &  {$42 \pm 1$}  &  {$83 \pm 3$}  &  {$55 \pm 5$}  \\
		MI & {$13.2 \pm 0.4$} &  {$50 \pm 3$}  &  {$84 \pm 0$}  &  {$78 \pm 2$}  &  {$70 \pm 0$}  &  {$76 \pm 3$}  &  {$96$}  &  {$38 \pm 1$}  &  {$-$}  &  {$-$}  \\ 
		MMD$_{-s}$& {$15.8 \pm 0.5$} &  {$55 \pm 5$}  &  {$84 \pm 0$}  &  {$82 \pm 0$}  &  {$73 \pm 1$}  &  {$75 \pm 2$}  &  {$96 \pm 0$}  &  {$35 \pm 2$}  &  {$85 \pm 3$}  &  {$49 \pm 3$}  \\ 
		MMD$_f$ & {$15.8 \pm 0.5$} &  {$50 \pm 5$}  &  {$83 \pm 0$}  &  {$80 \pm 0$}  &  {$74 \pm 1$}  &  {$78 \pm 2$}  &  {$96 \pm 0$}  &  {$34 \pm 1$}  &  {$86 \pm 1$}  &  {$57 \pm 6$}  \\ 
		OT & {$14.4 \pm 0.4$} &  {$61 \pm 3$}  &  {$83 \pm 0$}  &  {$78 \pm 1$}  &  {$72 \pm 2$}  &  {$75 \pm 3$}  &  {$-$}  &  {$-$}  &  {$-$}  &  {$-$}  \\ 
		CAI & {$11.8 \pm 0.3$} &  {$48 \pm 9$}  &  {$84 \pm 0$}  &  {$81 \pm 3$}  &  {$73 \pm 1$}  &  {$75 \pm 2$}  &  {$96$}  &  {$38$}  &  {$85 \pm 2$}  &  {$51 \pm 4$}  \\ 
		UAI & {$-$} &  {$-$}  &  {$84 \pm 0$}  &  {$83 \pm 0$}  &  {$73 \pm 2$}  &  {$75 \pm 3$}  &  {$98$}  &  {$34$}  &  {$84 \pm 3$}  &  {$49 \pm 7$}  \\
		\textbf{ICL} (Ours) & {$16.6 \pm 0.1$} &  {$\boldsymbol{32 \pm 0}$}  &  {$83 \pm 0$}  &  {$\boldsymbol{75 \pm 2}$}  &  {$75 \pm 2$}  &  {$\boldsymbol{75 \pm 2}$}  &  {$96 \pm 0$}  &  {$\boldsymbol{33 \pm 1}$}  &  {$84 \pm 3$}  &  {$\boldsymbol{46 \pm 7}$}  \\ 
		\midrule\midrule
	\end{tabular*} 
	\captionsetup{justification=justified} 
	\caption{\label{tab:results} ICL achieves a better Adversarial Invariance Measure (A) relative to the baselines as indicated in bold. The Prediction Accuracy (P) / Reconstruction Error (R) for all the methods are comparable. We include the following baselines: (a) Unregularized setup, (b) MI \cite{moyer2018invariant}, (c) $\mmd_{-s}$, (d) $\mmd_{f}$, based on \cite{li2014learning}, (e) OT based regularizer (f) CAI \cite{xie2017controllable} (g) UAI \cite{NIPS2018_7756}. The symbol $(-)$ indicates that the baseline was not applicable for the dataset. } 
\end{table*} 

First, we apply ICL to the family of generative models based on VAEs.
Primarily we work with the setups \eqref{eq:vae_obj} and \eqref{eq:vib_obj}.

\textbf{Learning style information in MNIST Dataset.} 
We consider the problem of learning representations 
that preserve only the style information of the digit (e.g.,
slant of digit, thickness of stroke etc.) while being invariant to the digit label. 
We use the VAE setup from \eqref{eq:vae_obj}. {\bf \textit{Results.}}
Table \ref{tab:results} shows that ICL provides the best adversarial invariance amongst
all the methods. 
Except for CAI, the invariance provided by other methods are significantly 
worse in comparison to the unregularized case.
We reviewed this behavior in Section \ref{section:introduction} and
suspect that it is due to a high similarity between input examples of the same digit.
We also show the effect of large regularizer weight 
to explain this behavior. 
t-SNE plots in Figure \ref{fig:tsne_mnist} show clusters
and collapse of the latent space for KL and $\mmd_{-s}$. In comparison, ICL
has a uniform latent space,
which partly explains why it provides better invariance.

\textbf{Learning invariant representation for Fairness Datasets.} 
Next, we consider the problem of learning representations that are invariant to the extraneous variable 
which may be ``protected'' in fair classification models. 
The intuition is that such invariant features should help downstream fair algorithms that
depend on these representations. 
We use the Adult and German datasets \cite{Dua:2019} for this task.
In Adult, the task is to predict if a person has over $\$50,000$ in savings, 
and the extraneous variable is Gender.
In German, the task is to predict if a person has a good credit score
and the extraneous variable is Age (binarized).
We use the preprocessing from \cite{moyer2018invariant}, and follow the VIB \eqref{eq:vib_obj} setup.
{\bf \textit{Results.}}  For Adult (Table \ref{tab:results}), all methods show 
comparable prediction accuracy and
ICL gives the best invariance. For German, ICL is amongst the methods
with best adversarial invariance (Table \ref{tab:results}) and provides best predictive accuracy.
Accuracy higher than unregularized case suggests that removal of Age 
assists the downstream task.

\subsection{Discriminative Model families}
Next we apply ICL to discriminative models \eqref{eq:sup_obj} which are parameterized using a deep neural network such as ResNet18 \cite{he2016identity}.
We seek to make representations at an 
internal layer of the network invariant, and 
so some VAE based baselines are not directly applicable.

\textbf{Invariance w.r.t. continuous extraneous attribute for Adult Dataset.}
For Adult dataset, we evaluate ICL in the context of \textit{age}, a
continuous $c$ variable. {\bf \textit{Results.}} In Table \ref{tab:results_adult}, we see that ICL 
provides a significantly better invariance in comparison to the baselines.
Since continuous attributes are common in the fairness literature
as well as in the context of applications in scientific disciplines,
we believe this experiment shows the viability of ICL's use in this setting.
%is an ongoing important problem in fairness literature \cite{}
\begin{table}[!b] 
	\footnotesize
	\captionsetup{justification=centering} 
	\setlength\tabcolsep{0pt} % make LaTeX figure out width of inter-column spaces
	\begin{tabular*}{\columnwidth}{l @{\extracolsep{\fill}}
			*{10}{S[table-format=1.4]}}
		\midrule\midrule
		Dataset: Adult with Age & {P $\uparrow$} & {A$_{\text{MSE}}$ $\uparrow$} \\ 
		\midrule\midrule
		Unregularized & {$83 \pm 0$} &  {$112 \pm 1$}  \\
		CAI \cite{xie2017controllable} & {$82 \pm 2$} &  {$129 \pm 10$}   \\ 
		UAI \cite{NIPS2018_7756} & {$84 \pm 0$} &  {$114 \pm 2$}   \\
		\textbf{ICL} (Ours) & {$83 \pm 0$} &  {$\boldsymbol{161 \pm 15}$} \\ 
		\midrule\midrule
	\end{tabular*} 
	\captionsetup{justification=justified} 
	\caption{\label{tab:results_adult} We study the continuous extraneous variable setting with the Adult dataset and \textit{Age} as the extraneous attribute. We find that  ICL attains a better Adversarial Invariance Measure (A$_{\text{MSE}}$$\uparrow$) compared to the baselines applicable in this setting.}
\end{table} 

\textbf{Rotation invariance for MNIST-ROT.} {This is a variant on MNIST dataset from \cite{NIPS2018_7756} where each digit 
	is randomly rotated by an angle $\in \{0, \pm 22.5^\circ, \pm 45^\circ\}$. 
	The task is to achieve invariance wrt rotation for predicting the digit label. {\bf \textit{Results.}} In Table \ref{tab:results}, we see that while being comparable
	in predictive accuracy, ICL provides the best adversarial
	invariance against rotation.}

\textbf{Predicting disease status while controlling for scanner confounds (ADNI dataset (adni.loni.usc.edu)).} We finally show the effectiveness
of ICL for predicting, using brain imaging data, whether an individual
has Alzheimer's disease (AD) or is a healthy control subject (CN). 
%on the problem of predicting if a person as Alzheimer's or not.
Our pre-processed dataset consists of about $449$ brain MRI scans of patients --
of note here is that because the acquisitions are performed at different
sites, the scanner manufacturers are different (e.g., Siemens, GE) \cite{giannelli2010dependence}.
While the pulse sequences for the scans are standardized, because
of differences in the magnetic coils and other factors, it is
not realistic for the images to be completely harmonized. If a handful
of coarse region of interest (ROI)
summaries are obtained from the images via some pre-processing methods (such as
Freesurfer),
one may expect some immunity to scanner specific artifacts. But if the goal
is to maximize performance using whole brain images, it becomes
difficult to discourage an off-the-shelf CNN model from picking up scanner specific
artifacts, especially if the demographics of the subjects are not perfectly
matched across sites. 
%different sites which use different scanners for obtaining the scans.
Here, we use the imaging protocol (site/scanner) as the categorical variable
we wish to control for. 
%used for scanning should not influence the models prediction and is a extraneous variable here.
While more specialized models can be used if desired to further improve
performance, 
we trained a simple ResNet-18 based model and
use the output of the last hidden layer as the latent representation. The response
variable was disease status: AD or CN. Since the dataset is small, 
the results are reported over five random training validation split.
{\bf \textit{Results.}}  We find that for this challenging setting,
ICL gives the best adversarial invariance (Table \ref{tab:results}) while also providing better predictive accuracy
than the unregularized model.

{{\bf Discussion on ICL's use for downstream tasks:}
	Our experiments show that ICL is effective in preventing an adversarial module from identifying the extraneous attribute from the latent representations. This would prevent the downstream models from using these extraneous features for prediction. These representations appear to be beneficial for use within fair algorithms. For some of our experiments, we observe that invariance leads to improved prediction accuracy of 
	the downstream task. We also provide a real world application where invariant representations help in
	pooling data from multiple sites, relevant in scientific studies.
}

\section{Conclusions}

Whether for compliance with legislative policies that forbid
preferential treatment (positive or negative) based on
protected attributes or to derive some level of immunity
to systematic variations when pooling data in a large
observational study spanning participating institutions,
it is clear that the need for invariant representations
within a sub-class of problems in machine learning 
will continue to grow and be broadly adopted. 
The form of ICL described here exhibits a number of
desirable properties and empirical behavior in scenarios/datasets that have been
described in the literature. While contrastive
losses are not new, recent results shed light
on when one may be able to characterize their
performance provably. As this literature continues to grow,
at least some of the findings will translate to and help
inform additional invariance properties afforded by ICL
and its variants.

\section{Broader Impact}

The general idea of invariant representations is closely tied to ongoing research on fair algorithms.
In that sense, ICL and other measures for invariance can enable the 
design of methods with a more desirable behavior, if the protected variables are
appropriately controlled for. Such strategies 
can also facilitate pooling of data from multiple sites, and help
answer important scientific questions that may not be possible to answer
with small sized datasets. Controlling
for undesirable observed variables
%is likely to become
will be an important consideration in a number of biomedical
applications where deep learning models are getting increasingly
adopted.

\begin{comment}
With any method that claims to offer a property ``X'', its adoption
can lead to significant blind spots. One example is complacency -- in the
absence of appropriate checks and balances, it can be expected
that one would encounter undesirable consequences. In specific situations such as
prison sentencing recommendations, a model that claims to be agnostic of race, but offers
weak invariance at best, can have rather serious ramifications on an individual. Clearly,
other mechanisms need to be put into place to identify, diagnose
and if possible, mitigate such occurrences. Note that the model's behavior
could be unintended (undesirable but nonetheless,
innate to its functionality) or a result of coordinated nefarious activities.
Similarly, we believe that applications in the biomedical sciences
need a great deal of care and solutions must include
human supervision at all stages of the pipeline. If a model invariant
to mammography scanner artifacts does not work as intended,
it is likely that women undergoing a scan on one type of scanner
will be disproportionately recommended for biopsies, causing
unnecessary stress for the patient and their families.
Research on explainable and interpretable methods together with the involvement
of an expert may have an outsized role to play in these settings. 
\end{comment}

\section{Acknowledgments}
        {The authors are grateful to Eric Huang for help and suggestions. Research
          supported by NIH R01 AG062336, NSF CAREER RI\#1252725, NSF CCF \#1918211,
  NIH RF1 AG059312, NIH RF1 AG05986901 and UW CPCP (U54 AI117924). Sathya Ravi was also supported by UIC-ICR
start-up funds. Correspondence should be directed to Ravi or Singh.}

% Working version. Would be merged with AKASH.tex

\section{Appendix}

%{\color{blue}

% Definition of Rw
\begin{definition}
	Let $p, q$ be two distributions and $w(x,y)$ be the interaction energy potential. Then the distributional interaction energy functional between distributions $p,q$ is defined as 
	\begin{equation}
	R_w(p, q) = \underset{{\substack{z \sim p \\ z' \sim p} }}{\mathbb{E}}w(z,z') + \underset{{\substack{z \sim q \\ z' \sim q} }}{\mathbb{E}} w(z,z') + 2\underset{{\substack{z \sim p \\ z' \sim q} }}{\mathbb{E}} w(z,z')
	\label{eq:rw_def}
	\end{equation}
	where the potential $w(x,y)$ is chosen suitably for different applications.
\end{definition}
See equation 1.1 in \citep{carrillo2012gradient}) for reference to interaction energy functional.

% Proof of Lemma 1
\subsection{Proof of Lemma 1}
\begin{proof}
	Recall the definition of $\icl(z,c)$ from \eqref{eq:icl_joint}.
	For binary extraneous variable $c$,  we have $\mathcal{N}_{\delta}(c) = \{c\}$. 
	Using this to simplify $\mathbbm{1}(c' \in \mathcal{N}_\delta(c)) = \mathbbm{1}_{(c = c')}$ 
	and $\mathbbm{1}(c' \notin \mathcal{N}_\delta(c)) = \mathbbm{1}_{(c \ne c')}$ and plugging in 
	\eqref{eq:icl_joint}, we obtain
	\begin{equation}
	\begin{split}
	\icl(z, c) = \underset{{\substack{(z,c) \sim p(z,c) \\ (z',c') \sim p(z', c')} }}{\mathbb{E}}\mathbbm{1}_{(c = c')} f(z,z') +\mathbbm{1}_{(c \ne c')} s(z,z')
	\end{split}		
	\label{eq:icl_pf_1}
	\end{equation}	
	Next we introduce the functions $g$ and $w$ used in the Lemma
	\begin{equation}
	g(z,z') = (f(z,z') - s(z,z'))/8
	\label{eq:g_func}
	\end{equation}
	\begin{equation}
	w(z,z') = (f(z,z') + s(z,z'))/8
	\label{eq:w_func}
	\end{equation}
	Using \eqref{eq:g_func} and \eqref{eq:w_func} in \eqref{eq:icl_pf_1} gives us 
	\begin{equation}
	\begin{split}
	\frac{\icl(z, c)}{4}  &= \underset{{\substack{(z,c) \sim p(z,c) \\ (z',c') \sim p(z', c')} }}{\mathbb{E}}\bigg[\bigg(\mathbbm{1}_{(c = c')}-\mathbbm{1}_{(c \ne c')}\bigg) g(z,z')\bigg] \\
	&\quad \quad + \underset{{\substack{(z,c) \sim p(z,c) \\ (z',c') \sim p(z', c')} }}{\mathbb{E}}[w(z,z')]
	\end{split}		
	\label{eq:icl_pf_2}
	\end{equation}
	Using law of total expectation we write \eqref{eq:icl_pf_2} as
	\begin{equation}
	\begin{split}
	\frac{\icl(z, c)}{4} &= \underset{{c,c' \sim p(c) }}{\mathbb{E}}\bigg[\big(\mathbbm{1}_{(c = c')}-\mathbbm{1}_{(c \ne c')}\big) \underset{{\substack{z \sim p(z|c) \\ z' \sim p(z| c')} }}{\mathbb{E}} [g(z,z')]\bigg] \\ 
	& \quad \quad \quad \quad + \underset{{c,c' \sim p(c) }}{\mathbb{E}}\bigg[  \underset{{\substack{z \sim p(z|c) \\ z' \sim p(z| c')} }}{\mathbb{E}} [w(z,z')] \bigg]\\
	\end{split}		
	\end{equation}
	Since $p(c=0) = 1/2$, and using $p_0$ and $p_1$ to denote the conditional distributions $p(z|c=0)$ and $p(z|c=1)$ respectively, the expectation is expanded to get
	\begin{equation}
	\begin{split}
	\icl(z, c) &=  \underbrace{\underset{{\substack{z \sim p_0 \\ z' \sim p_0} }}{\mathbb{E}} g(z,z') + \underset{{\substack{z \sim p_1 \\ z' \sim p_1} }}{\mathbb{E}}g(z,z') - 2\underset{{\substack{z \sim p_0 \\ z' \sim p_1} }}{\mathbb{E}} g(z,z')}_{\mmd_g} \\
	& + \underbrace{\underset{{\substack{z \sim p_0 \\ z' \sim p_0} }}{\mathbb{E}} w(z,z') + \underset{{\substack{z \sim p_1 \\ z' \sim p_1} }}{\mathbb{E}} w(z,z') + 2\underset{{\substack{z \sim p_0 \\ z' \sim p_1} }}{\mathbb{E}} w(z,z')}_{R_w}\\
	&= \mmd_g(p_0, p_1) + R_w(p_0, p_1)
	\end{split}		
	\end{equation}
	where $R_w(p, q)$ is defined in \eqref{eq:rw_def}.
	
\end{proof}

% Generalization of Lemma 1 for multiple variables
\subsection{Generalization of Lemma 1}
We next show that Lemma \ref{lem:icl-r-mmd} can be generalized to multi-class setting when $c$ is (discrete) uniformly distributed.
\begin{lemma} \label{lem:icl-r-mmd-gen} Assume that the extraneous variable $c$ is discrete with $c\in \{1,...,m\}$ and is uniformly distributed, $p(c=i)=1/m$. Then there exist a positive definite kernel $g$ and an interaction energy functional $R_w$  (see \eqref{eq:rw_def}) such that the following equality holds:
	\begin{align}
	\icl(z, c) = \sum_{ \substack{i,j \in \{1,...,m\}, \\ i < j} }\mmd_{g}(p_i, p_j) + R_w(p_i, p_j),\label{eq:icl-r-mmd-gen}
	\end{align}
	where $p_i$ denote the conditional distributions $p(z|c=i)$.
\end{lemma}
\begin{proof}
	The proof proceeds on similar lines as the proof of Lemma \ref{lem:icl-r-mmd}. We introduce new functions $g$ and $w$ for the multi-class setting as 
	\begin{equation}
	g(z,z') = \frac{1}{m^2} \bigg(\frac{f(z,z')}{m-1} - s(z,z')\bigg)
	\label{eq:g_func_gen}
	\end{equation}
	\begin{equation}
	w(z,z') = \frac{1}{m^2} \bigg(\frac{f(z,z')}{m-1} + s(z,z')\bigg)
	\label{eq:w_func_gen}
	\end{equation}
	Using law of total expectation we write \eqref{eq:icl_pf_1}
	\begin{equation}
	\begin{split}
	&\icl(z,c) = \underset{{c,c' \sim p(c)}}{\mathbb{E}} \bigg[\mathbbm{1}_{(c = c')} \underset{{\substack{ z \sim p(z|c) \\ z' \sim p(z| c') }}}{\mathbb{E}} [f(z,z')] \\
	&\quad \quad \quad \quad \quad \quad  \quad \quad \quad +\mathbbm{1}_{(c \ne c')} \underset{{\substack{ z \sim p(z|c) \\ z' \sim p(z| c') }}}{\mathbb{E}} [s(z,z')] \bigg]\\
	&= \frac{1}{m^2} \sum_{i,j \in \{1,...,m\}} \mathbbm{1}_{(i = j)} \underset{{\substack{ z \sim p_i \\ z' \sim p_j }}}{\mathbb{E}} f(z,z') +\mathbbm{1}_{(i \ne j)} \underset{{\substack{ z \sim p_i \\ z' \sim p_j }}}{\mathbb{E}} s(z,z') \\
	& \qquad \qquad \qquad \qquad \{\text{Expanding the expectation}\}
	\end{split}		
	\label{eq:icl_pf_gen_2}
	\end{equation}
	Using \eqref{eq:g_func_gen}, \eqref{eq:w_func_gen} in above and rearranging gives
	\begin{equation}
	\begin{split}
	&\icl(z, c) = \\
	&\sum_{\substack{i,j \in \{1,...,m\} \\ i < j}} \bigg(\underbrace{ \underset{{\substack{ z \sim p_i \\ z' \sim p_i }}}{\mathbb{E}}g(z,z') + \underset{{\substack{ z \sim p_j \\ z' \sim p_j }}}{\mathbb{E}} g(z,z') -2 \underset{{\substack{ z \sim p_i \\ z' \sim p_i }}}{\mathbb{E}} g(z,z')}_{\mmd_g}  \\ 
	& \qquad\quad + \underbrace{ \underset{{\substack{ z \sim p_i \\ z' \sim p_i }}}{\mathbb{E}} w(z,z') + \underset{{\substack{ z \sim p_j \\ z' \sim p_j }}}{\mathbb{E}} w(z,z') + 2 \underset{{\substack{ z \sim p_i \\ z' \sim p_i }}}{\mathbb{E}}  w(z,z')}_{R_w} \bigg)\\
	&= \sum_{\substack{i,j \in \{1,...,m\} \\ i < j}} \mmd_{g}(p_i, p_j) + R_w(p_i, p_j)
	\end{split}		
	\label{eq:icl_pf_gen_3}
	\end{equation}
\end{proof}

% Proof of Lemma 2
\subsection{Proof of Lemma 2}
\begin{proof}
	
	Define $\mathcal{L}(f) := \mathbb{E}_{(z,c)}[(b(z) - c)^2]$ as the MSE of adversary $b$. Next we introduce imaginary samples $(z',c') \sim p(z,c)$ and have following 
	\begin{equation}
	\begin{split}
	2\mathcal{L}(f) &= 2\underset{{(z,c)\sim p(z,c)}}{\mathbb{E}}[(b(z)-c)^2] \\
	&= \underset{{\substack{(z,c)\sim p(z,c) \\ (z',c')\sim p(z,c)}}}{\mathbb{E}}[(b(z)-c)^2 + (b(z')-c')^2]
	\end{split}
	\end{equation}
	We hide the subscript in expectation to simplify the notation. 
	\begin{equation}
	\begin{split}
	&4\mathcal{L}(f) = 2\mathbb{E}[(b(z)-c)^2 + (b(z')-c')^2]\\
	&\ge \mathbb{E}\big[\big(b(z)-c - b(z')+c'\big)^2\big] \;\;\; \{(a^2 + b^2) \ge {(a-b)^2}/{2} \}\\
	&=  \mathbb{E}\bigg[\bigg((c'-c) - \big(b(z') - b(z)\big)\bigg)^2\bigg] \quad \{\text{Rearranging}\}\\
	&\ge \mathbb{E}\big[\big(|c'-c| - |b(z')-b(z)|\big)^2\big] \;\; \{(a-b)^2 \ge (|a|-|b|)^2\}\\
	&=  \mathbb{E}\big[\big(\mathbbm{1}_{(c' \in \mathcal{N}_\delta(c))}+\mathbbm{1}_{(c' \notin \mathcal{N}_\delta(c))}\big)
	\big(|c'-c| - |b(z')-b(z)|\big)^2\big]\\
	%	\end{split}
	%	\end{equation}
	%	
	%	
	%	\begin{equation}
	%	\begin{split}
	&\ge \mathbb{E}\big[\mathbbm{1}_{(c' \notin \mathcal{N}_\delta(c))}\big(|c'-c| - |b(z')-b(z)|\big)^2\big]\\
	&= \mathbb{E}\big[\big(\mathbbm{1}_{(c' \notin \mathcal{N}_\delta(c))}|c'-c| - \mathbbm{1}_{(c' \notin \mathcal{N}_\delta(c))}|b(z')-b(z)|\big)^2\big]\\
	&\ge \mathbb{E}\big[\big(\mathbbm{1}_{(c' \notin \mathcal{N}_\delta(c))}|c'-c| - \mathbbm{1}_{(c' \notin \mathcal{N}_\delta(c))}|b(z')-b(z)|\big)^2\big] \\
	& \{\text{Using $\mathbb{E}[X^2] \ge \mathbb{E}[X]^2$, $\mathbb{E}[X-Y] = \mathbb{E}[X] - \mathbb{E}[Y]$}\}\\
	&\ge  \mathbb{E}\bigg[\bigg(\mathbbm{1}_{(c' \notin \mathcal{N}_\delta(c))}|c'-c| - \mathbbm{1}_{(c' \notin \mathcal{N}_\delta(c))}|b(z')-b(z)|\bigg)\bigg]^2 \\
	&\ge \bigg( \underbrace{\mathbb{E}[\mathbbm{1}_{(c' \notin \mathcal{N}_\delta(c))}|c'-c|]}_{\rom{1}} -\underbrace{ \mathbb{E}[\mathbbm{1}_{(c' \notin \mathcal{N}_\delta(c))}|b(z')-b(z)|]}_{\rom{2}} \bigg)^2 
	\end{split}
	\label{eq:l2_pf_1}
	\end{equation}
	We lower bound $\rom{1}$ using the encoding $\mathcal{N}_\delta(c) = \{c' : |c-c'| \le \delta\}$ as follows,  
	\begin{equation}
	\begin{split}
	\rom{1} &= \mathbb{E}[\mathbbm{1}(c' \notin \mathcal{N}_\delta(c))|c'-c|] \\
	&\ge \delta \mathbb{E}[\mathbbm{1}(c' \notin \mathcal{N}_\delta(c))]
	= \delta P_{c,c'}(|c-c'| > \delta) 
	= \delta \rho
	\end{split}
	\label{eq:l2_pf_2}
	\end{equation}
	Next we use the fact that $\icl(z,c) < \epsilon$ to obtain following 
	\begin{equation}
	\begin{split}
	&\icl(z,c) < \epsilon \\
	&\implies  \mathbb{E}[\mathbbm{1}_{(c' \in \mathcal{N}_\delta(c))} f(z,z') + \mathbbm{1}_{(c' \notin \mathcal{N}_\delta(c))} s(z,z')] < \epsilon \\
	&\implies \mathbb{E}[ \mathbbm{1}_{(c' \notin \mathcal{N}_\delta(c))} s(z,z')] < \epsilon\\
	&\implies \mathbb{E}[ \mathbbm{1}_{(c' \notin \mathcal{N}_\delta(c))} d^2(z,z')] < \epsilon \;\;  \{\text{As } s(z,z')=d^2(z,z')\}\\
	&\implies \mathbb{E}[ \mathbbm{1}_{(c' \notin \mathcal{N}_\delta(c))} d(z,z')] < \sqrt{\epsilon} \;\;  \{\text{As } \mathbb{E}[X] \le \sqrt{\mathbb{E}[X^2]}\} \label{eq:l2_pf_3}
	\end{split}
	\end{equation}
	Since $b$ is L-lipschitz, $|b(z)-b(z')| \le L \; d(z,z')$, which allows us to upper bound $\rom{2}$ as 
	\begin{equation}
	\begin{split}
	\rom{2} &= \mathbb{E}[\mathbbm{1}_{(c' \notin \mathcal{N}_\delta(c))}|b(z')-b(z)|]\\
	&\le L\; \mathbb{E}[\mathbbm{1}_{(c' \notin \mathcal{N}_\delta(c))} d(z,z')] \quad \{\text{Using lipschitz definition}\}\\
	&< L\sqrt{\epsilon} \quad \{\text{Using \eqref{eq:l2_pf_3}}\} 
	\end{split}
	\label{eq:l2_pf_4}
	\end{equation}
	We choose $\epsilon$ such that $\epsilon < \delta^2\rho^2/L^2$. Note that there exists a $\alpha$ such that $\icl^{\delta}_{\alpha, \beta}(z,c) < \epsilon$ for this choice. 
	This allows us to use the lower and upper bounds of $\rom{1}$ and $\rom{2}$ respectively from \eqref{eq:l2_pf_2} and \eqref{eq:l2_pf_4} in \eqref{eq:l2_pf_1} to give
	\begin{equation}
	\mathcal{L}(f) \ge (\delta \rho - L\sqrt{\epsilon})^2/4
	\end{equation}
\end{proof}

% More training details
\subsection{Detailed Setup for Applications}
{\bf (a) Details on adversary}. We follow \cite{xie2017controllable} for training the adversary used for reporting invariance. We use a three-layered FC network with batch normalization and train it using Adam.
For the MNIST-ROT experiment, we follow the setup of \cite{NIPS2018_7756}.\\
{\bf (b) Evaluation methodology}. We evaluate the frameworks in terms of
task accuracy/reconstruction error and adversarial invariance on an unseen test set.
The ADNI dataset is very small and hence for this dataset we use five 
fold random training validation splits to report the mean and standard deviation.
For all other experiments, the mean and standard deviation are reported on an
unseen test set for ten random runs, except when
quoting results from previous works.\\
{\bf (c) Hyperparameter selection.} The hyperparameter selection is done on a separate
validation split such that on this set the model achieves the best adversarial 
invariance while the task accuracy remains within 
$5\%$ of the unregularized model for supervised tasks and within
$5$ points of the unregularized model for unsupervised tasks. 
%This allows us to compare 
%frameworks for the scenario when they impose invariance while
%also providing reasonable downstream task performance. 
For the baselines,  
we grid search the best regularization weight in powers of ten and select the one
with best invariance on validation set.  
For some of the experiments, we found 
it useful to initialize the regularization weight to a smaller value ($0.01$ times the regularizer weight) and
multiplicatively update it (with factor $1.5$) every epoch till it reaches the best found regularization weight. 
The  same update rule is used for all the baselines. 
\\
%{\bf (d) No $\mmd_g$ baseline.} Note that, we do not find $\mmd_g$ to be a suitable baseline for our experiments. When the terms in ${\rm MMD}_g$ are viewed from a
%dynamical systems perspective, the repulsion term always pushes
%the intraclass representations further away from each other
%independent of the distance. So,
%increasing the regularization penalty for this regularizer leads
%the latent space to expand greatly which hurts performance. \\
{\bf (d) ICL parameters}. For identifying ICL parameters $\alpha$, $\beta$ and $\delta$, we perform
simple grid search in powers of ten and its multiples of two and five. The $\delta$ hyperparameter is only relevant 
for the case of continuous extraneous attribute. For the continuous case, we normalize
the extraneous variable to be in $[0, 1]$ and search the $\delta$ parameter from multiples of $0.05$.

%}

% References and End of Paper
% These lines must be placed at the end of your paper
\bibliography{egbib}

%\newpage

%\input{temp}

\end{document}